%% file: main.tex
\documentclass{article}

\usepackage{microtype}
\usepackage{graphicx}
\usepackage{subfigure}
\usepackage{booktabs} 

\usepackage{arxiv}
\usepackage[margin=1.48in]{geometry}
\usepackage[parfill]{parskip}

\usepackage[inline]{enumitem}
\usepackage{amsfonts} 
\usepackage{amsmath}
\usepackage{algorithm}
\usepackage{algorithmic}
\usepackage{float}
\usepackage{wrapfig}
\usepackage{natbib} 
\usepackage{mathtools}
\usepackage{amsthm}
\usepackage{color}
\usepackage{bbm}
\allowdisplaybreaks

\usepackage{hyperref}

\theoremstyle{definition}
\newtheorem{theorem}{Theorem}[section]
\newtheorem{definition}[theorem]{Definition}

\newtheorem{proposition}[theorem]{Proposition}
\newtheorem{corollary}[theorem]{Corollary}

\newtheorem{assumption}[theorem]{Assumption}
\newtheorem{construction}[theorem]{Construction}

\newcommand{\abs}[1]{\left\lvert#1 \right\rvert}
\newcommand{\norm}[1]{\left\lVert#1 \right\rVert}

\newcommand{\err}{\mathcal{E}}

\newcommand{\RR}{\mathbb{R}}

\newcommand{\EE}[2]{\mathbb{E}_{#1}\left[#2\right]}
\newcommand{\HH}{\mathcal{H}}
\newcommand{\A}{\mathcal{A}}
\newcommand{\B}{\mathcal{B}}
\newcommand{\M}{\mathcal{M}}
\newcommand{\Normal}{\mathcal{N}}
\newcommand{\W}{\mathcal{W}}

\newcommand\todo[1]{\textcolor{red}{[\textbf{TODO}: #1]}}

\DeclarePairedDelimiterX{\inp}[2]{\langle}{\rangle}{#1, #2}
\newcommand{\I}[1]{\mathbbm{1}\left\{#1\right\}}
\newcommand{\addeq}{\addtocounter{equation}{1}\tag{\theequation}}

\newcommand{\lsp}{\mathcal{Z}}
\newcommand{\fsp}{\Phi}
\newcommand{\rsp}{\mathcal{X}}
\DeclarePairedDelimiterX{\supp}[1]{\mathrm{Supp}(}{)}{#1}
\DeclarePairedDelimiterX{\dom}[1]{\mathrm{dom}(}{)}{#1}

\newcommand{\risk}{r}
\newcommand{\dz}{\mathrm{d}z}
\newcommand{\dx}{\mathrm{d}x}

\title{Domain Adaptation with Asymmetrically-Relaxed Distribution Alignment}
\date{}

\author{
Yifan Wu
\\
Carnegie Mellon University\\
\texttt{yw4@cs.cmu.edu}\\
 \And
Ezra Winston \\
Carnegie Mellon University\\
\texttt{ewinston@cs.cmu.edu} \\
\AND
Divyansh Kaushik \\
Carnegie Mellon University\\
\texttt{dkaushik@cs.cmu.edu} \\
 \And
Zachary Lipton \\
Carnegie Mellon University\\
\texttt{zlipton@cmu.edu} 
}

\begin{document}
\maketitle

\begin{abstract}

\input{sections/abstract}
\end{abstract}


\input{sections/intro}
\input{sections/background}
\input{sections/motivation}

\input{sections/theory}
\input{sections/algorithms}
\input{sections/experiments}
\input{sections/related}
\input{sections/conclusion}

\subsection*{Acknowledgments}
This work was made possible by a generous grant from the Center for Machine Learning and Health, a joint venture of Carnegie Mellon University, UPMC, and the University of Pittsburgh, in support of our collaboration with Abridge AI to develop robust models for machine learning in healthcare. We are also supported in this line of research by a generous faculty award from Salesforce Research.

\bibliography{main}
\bibliographystyle{icml2019}

\appendix
\onecolumn
\input{sections/appendix}

\end{document}

%% file: sections/abstract.tex
Domain adaptation addresses the common problem
when the \emph{target} distribution generating our test data
drifts from the \emph{source} (training) distribution.
While absent assumptions, domain adaptation is impossible, 
strict conditions, e.g. \emph{covariate} or \emph{label} shift, 
enable principled algorithms.
Recently-proposed domain-adversarial approaches
consist of aligning source and target encodings,
often motivating this approach as minimizing two (of three) terms
in a theoretical bound on target error. 
Unfortunately, this minimization can cause arbitrary increases in the third term,
e.g. they can break down under shifting label distributions.
We propose \emph{asymmetrically-relaxed distribution alignment},
a new approach that overcomes some limitations of standard domain-adversarial algorithms.
Moreover, we characterize precise assumptions 
under which our algorithm is theoretically principled
and demonstrate empirical benefits on both synthetic and real datasets.

%% file: sections/intro.tex
\section{Introduction}

Despite breakthroughs in supervised deep learning
across a variety of challenging tasks,
current techniques depend precariously on the i.i.d. assumption.
Unfortunately, real-world settings often demand 
not just generalization to \emph{unseen examples} 
but robustness under a variety of shocks 
to the data distribution.
Ideally, our models would leverage unlabeled test data,
adapting in real time to produce improved predictions. 
\emph{Unsupervised domain adaptation} formalizes this problem
as learning a classifier from labeled \emph{source domain} data 
and unlabeled data from a \emph{target domain},
to maximize performance on the target distribution.

Without further assumptions, guarantees of target-domain accuracy
are impossible \citep{ben2010impossibility}.
However, well-chosen assumptions can make possible 
algorithms with non-vacuous performance guarantees.
For example, under the \emph{covariate shift} assumption \citep{heckman1977sample, shimodaira2000improving}, 
although the input marginals can vary 
between source and target ($p_S(x) \neq p_T(x)$), 
the conditional distribution of the labels (given features)
exhibits invariance across domains ($p_S(y|x) = p_T(y|x)$).
Some consider the reverse setting \emph{label shift} 
\citep{saerens2002adjusting, zhang2013domain, lipton2018detecting},
where although the label distribution shifts ($p_S(y) \neq p_T(y)$),
the class-conditional input distribution is 
invariant ($p_S(x|y) = p_T(x|y)$).
Traditional approaches to both problems 
require the source distributions' support 
to cover the target support,
estimating adapted classifiers via 
importance-weighted risk minimization 
\citep{shimodaira2000improving, huang2007correcting, gretton2009covariate, yu2012analysis, lipton2018detecting}.

Problematically, assumptions of contained support 
are violated in practice.
Moreover, most theoretical analyses do not guaranteed target accuracy 
when the source distribution support does not
cover that of the target. 
A notable exception, \citet{ben2010theory} leverages 
capacity constraints on the hypothesis class
to enable generalization to out-of-support samples.
However, their results 
(i) do not hold for high-capacity hypothesis classes, e.g., neural networks;
and (ii) do not provide intuitive interpretations 
on what is sufficient to guarantee a good target domain performance.

A recent sequence of deep learning papers have proposed empirically-justified
adversarial training schemes aimed at practical problems with non-overlapping supports
\citep{ganin2016domain, tzeng2017adversarial}.
Example problems include generalizing from gray-scale images to colored images
or product images on white backgrounds to photos of products in natural settings.
While importance-weighting solutions are useless here 
(with non-overlapping support, weights are unbounded),
\emph{domain-adversarial networks} \citep{ganin2016domain} 
and subsequently-proposed variants report 
strong empirical results on a variety of image recognition challenges.

The key idea of domain-adversarial networks 
is to simultaneously minimize the source error
and align the two distributions in representation space.
The scheme consists of an encoder, 
a \emph{label classifier}, and a \emph{domain classifier}.
During training, the \emph{domain classifier} is optimized
to predict each image's domain given its encoding.
The \emph{label classifier} is optimized to predict labels from encodings (for source images).
The encoder weights are optimized for the twin objectives 
of accurate label classification (of source data)
and \emph{fooling} the domain classifier (for all data).

Although \citet{ganin2016domain} motivate their idea 
via theoretical results due to \citet{ben2010theory}, 
the theory is insufficient to justify their method.
Put simply, \citet{ben2010theory} bound the test error by a sum of three terms.
The domain-adversarial objective minimizes two among these,
but this minimization may cause the third term to increase.
This is guaranteed to happen when the label distribution shifts between source and target. 
Consider the case of cat-dog classification with non-overlapping support.
Say that the source distribution contains $50\%$ dogs and $50\%$ cats,
while the target distribution contains $25\%$ dogs and $75\%$ cats.
Successfully aligning these distributions in representation space 
requires the classifier to predict 
the same fraction of dogs and cats on source and target. 
If one achieves $100\%$ accuracy on the source data,
then target accuracy will be at most $75\%$
(Figure~\ref{fig:toy-exact}). 

\begin{figure}[t]
\centering
\subfigure[Exact matching]{
\includegraphics[width=0.45\columnwidth]{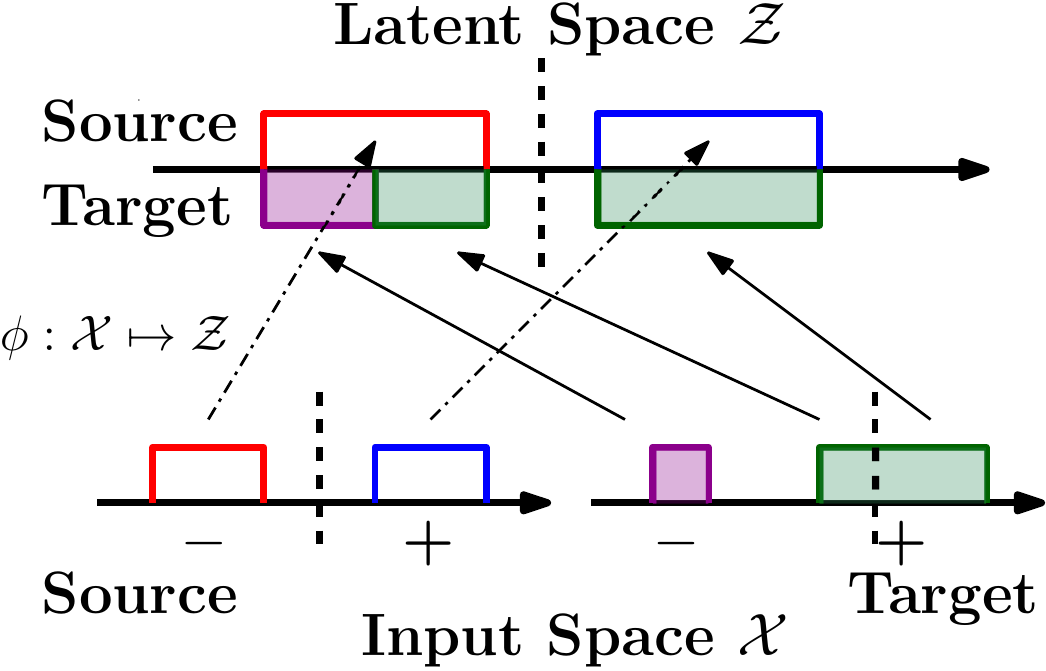}
\label{fig:toy-exact}
}
\subfigure[Relaxed matching]{
\includegraphics[width=0.45\columnwidth]{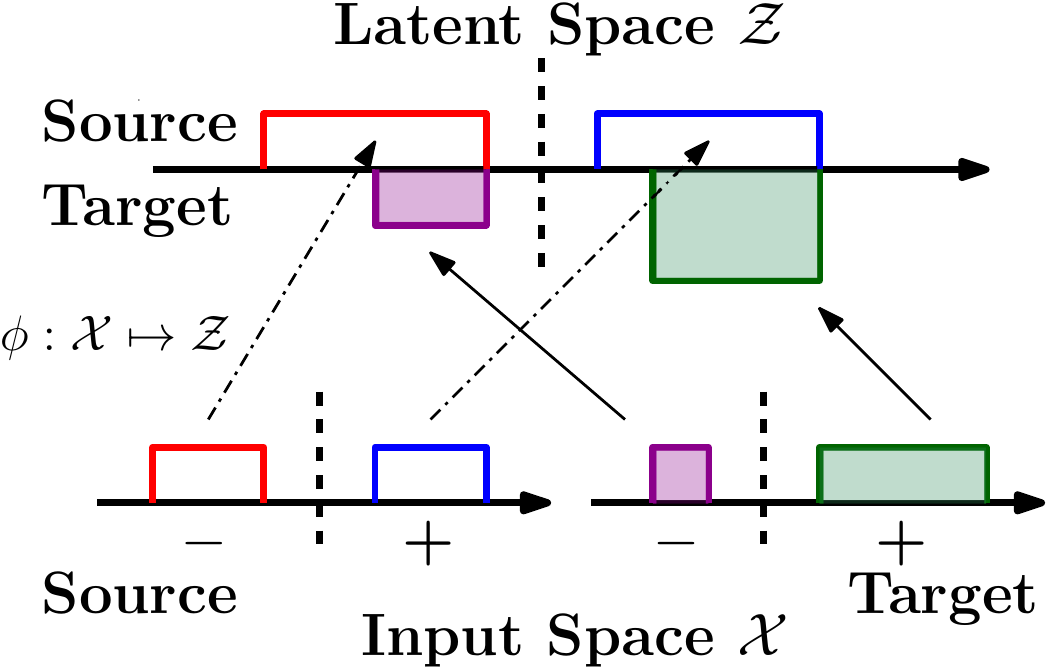}
\label{fig:toy-relaxed}
}
\caption{
(a) In order to match the latent space distributions exactly,
a model must map some elements of positive class in the target domain to some elements of negative class in the source domain. 
(b) A better mapping is achieved 
by requiring only that the source covers the target
in the latent space. 
}
\label{fig:toy}
\end{figure}

In this paper, we propose asymmetrically-relaxed distribution alignment,
a relaxed distance for aligning data across domains
that can be minimized without requiring latent-space distributions to match exactly. 
The new distance is minimized whenever the density ratios in representation space from target to source are upper bounded by a certain constant, 
such that the target representation support is contained in the source representation's. 
The relaxed distribution alignment need not lead 
to a poor classifier on the target domain under label distribution mismatch (Figure~\ref{fig:toy-relaxed}). 
We demonstrate theoretically that the relaxed alignment is sufficient 
for a good target domain performance under a concrete set of assumptions 
on the  data distributions. 
Further, we propose several practical ways to achieve the relaxed distribution alignment,
translating the new distance into adversarial learning objectives.
Empirical results on synthetic and real datasets show 
that incorporating our relaxed distribution alignment loss 
into adversarial domain adaptation 
gives better classification performance on the target domain.
We make the following key contributions: 
\begin{itemize}\setlength\itemsep{0em}
\item We propose an asymmetrically relaxed distribution matching objective, 
overcoming the limitation of standard 
objectives under label distribution shift.
\item We provide theoretical analysis demonstrating 
that under a clear set of assumptions,
the asymmetrically relaxed distribution alignment 
can provide target-domain performance guarantees. 
\item We propose several distances that
satisfy the desired properties 
and are optimizable by adversarial training.
\item We empirically show that our asymmetrically relaxed distribution matching losses 
improve target performance 
when there is a label distribution shift 
in the target domain,
and perform comparably otherwise.
\end{itemize}


%% file: sections/background.tex
\section{Preliminaries}
\label{sec:pre}

We use subscripts $S$ and $T$ to distinguish 
between source and target domains, 
e.g., $p_S$ and $p_T$,
and employ the notation $U$ for statements 
that are true for any domain $U\in \{S, T\}$. 
For simplicity, we dispense with some rigorousness in notating probability measures.
For example, we use the terms measure and distribution interchangeably 
and assume that a density function exists when necessary
without explicitly stating the base measure and required regularity conditions. 
We use a single lowercase letter, e.g. $p$, 
to denote both the probability measure function 
and the probability density function: 
$p(x)$ is a density when the input $x$ is a single point 
while $p(C)$ is a probability when the input $C$ is a set.
We will use $\supp{p}$ to denote the support of distribution $p$, i.e.,
the set of points where the density is positive.
Similarly, for a function mapping $\phi$, $\phi(x)$ denotes an \emph{output} 
if $x$ is a point and $\phi(C)$ denotes the \emph{image} if $C$ is a set. 
The inverse mapping $\phi^{-1}$ always outputs a set (the inverse image)
regardless of whether its input is a point or a set.
We will also be less careful about the use of $\sup$ v.s. $\max$, $\inf$ v.s. $\min$ 
and ``everywhere'' v.s. ``almost everywhere''. 
$\I{\cdot}$ is used as the indicator function for statements 
that output $1$ if the statement is true and $0$ otherwise. 
For two functions $f$ and $g$ we use $f\equiv g$ 
to denote that $f(x)=g(x)$ for every input $x$.


\paragraph{Unsupervised domain adaptation
}
For simplicity,
we address the binary classification scenario. 
Let $\rsp$ be the input space and $f: \rsp \mapsto \{0, 1\}$ 
be the (domain-invariant) 
ground truth labeling function. 
Let $p_S$ and $p_T$ be the input distributions over $\rsp$ 
for source and target domain respectively. 
Let $\lsp$ be a latent space and $\fsp$ 
denote a class of mappings from $\rsp$ to $\lsp$.  
For a domain $U$, let $p_U^\phi(\cdot)$ 
be the induced probability distribution 
over $\lsp$ such that 
$p_U^\phi(C)=p_U(\phi^{-1}(C))$
for any $C\subset \lsp$. 
Given $z\in \lsp$ let  $\phi_U(\cdot|z)$ 
be the conditional distribution induced by $p_U$  
and $\phi$ such that 
$\int \dz p_U^\phi(z)\phi_U(x|z)=p_U(x)$
holds for all $x\in \rsp$.
Define $\HH$ to be a class of 
predictors
over the latent space $\lsp$, i.e., 
each $h\in \HH$ 
maps
from $\lsp$ to $\{0, 1\}$. 
Given a representation mapping $\phi \in \fsp$, classifier $h \in \HH$, and input $x\in \rsp$,
our prediction is $h(\phi(x))$.
The 
risk for a single input $x$ can be written as
$|h(\phi(x))-f(x)|$ and
the expected 
risk for a domain $U$ is 
\begin{align*}
\err_U(\phi, h)
& = \int \dx p_U(x)  \abs{h(\phi(x))-f(x)} \\
& \doteq \int \dz p_U^\phi(z) \abs{h(z) - f^\phi_U(z)} \\
& \doteq \int \dz p_U^\phi(z) \risk_U(z; \phi, h) \addeq\label{eq:risk}
\end{align*}
where we define a domain-dependent latent space labeling function 
$f^\phi_U(z)=\int \dx \phi_U(x|z) f(x)$ 
and the risk 
for a classifier $h$ as 
$\risk_U(z; \phi, h) = \abs{h(z) - f^\phi_U(z)} \in [0, 1]$.

We are interested in bounding the classification risk of a $(\phi, h)$-pair on the target domain:
\begin{align*}
& \err_T(\phi, h) 
 = \int \dz p_T^\phi(z) \risk_T(z; \phi, h)  = \err_S(\phi, h) \\
 & ~~ + \int \dz p_T^\phi(z) \risk_T(z; \phi, h) - \int \dz p_S^\phi(z) \risk_S(z; \phi, h) \\
& = \err_S(\phi, h) + \int \dz p_T^\phi(z) \left(\risk_T(z; \phi, h)- \risk_S(z; \phi, h)\right) \\
& ~~ + \int \dz \left(p_T^\phi(z) - p_S^\phi(z)\right) \risk_S(z; \phi, h) \,. 
\addeq \label{eq:tar-err-decomp}
\end{align*}
The second term in \eqref{eq:tar-err-decomp} 
becomes zero 
if the latent space labeling function
is domain-invariant.
To see this, we apply 
\begin{align*}
& \risk_T(z; \phi, h)- \risk_S(z; \phi, h) = \abs{h(z) - f^\phi_T(z)} - \abs{h(z) - f^\phi_S(z)}\\
& \le \abs{f^\phi_T(z)-f^\phi_S(z)}\,. \addeq\label{eq:triangle}
\end{align*}
The third term in \eqref{eq:tar-err-decomp} is zero 
when 
$p_T^\phi$ and $p_S^\phi$ are the same. 

In the \emph{unsupervised domain adaptation} setting, we have access to labeled source data $(x, f(x))$ 
for $x\sim p_S$ and unlabeled target data $x\sim p_T$, 
from which we can calculate\footnote{In this work we focus 
on how domain adaption are able to generalize across distributions with different supports 
so we will not talk about finite-sample approximations.} 
the first and third term in \eqref{eq:tar-err-decomp}. 
For $x\in \supp{p_T}\setminus \supp{p_S}$, 
we have no information about its true label $f(x)$ 
and thus $f^\phi_T(z)$ becomes inaccessible when $z=\phi(x)$ for such $x$. 
So the second term in \eqref{eq:tar-err-decomp} is not directly controllable.

\paragraph{Domain-adversarial learning }
\emph{Domain-adversarial} approaches focus on minimizing the first and third term in \eqref{eq:tar-err-decomp} jointly. 
Informally, these approaches 
minimize the source domain classification risk 
and the distance between the two distributions 
in the latent space:
\begin{align*}
\min_{\phi, h} ~ \err_S(\phi, h) + \lambda D(p_S^\phi, p_T^\phi) + \Omega(\phi, h) \,, \addeq\label{eq:obj-general} 
\end{align*}
where $D$ is a distance metric between distributions 
and $\Omega$ is a regularization term. 
Standard choices of $D$ such as a domain classifier (Jensen-Shannon (JS) divergence
\footnote{Per \citep{nowozin2016f},
there is a slight difference between JS-divergence and the original GAN objective  \citep{goodfellow2014generative}. We will use the term JS-divergence for the GAN objective.}
) \citep{ganin2016domain}, 
Wasserstein distance \citep{shen2018wasserstein} or Maximum Mean Discrepancy \citep{huang2007correcting} 
have the property that $D(p_S^\phi, p_T^\phi)=0$ if $p_S^\phi\equiv p_T^\phi$ and $D(p_S^\phi,p_T^\phi)>0$ otherwise. 
In the next section, we will show that minimizing \eqref{eq:obj-general} with such $D$ 
will lead to undesirable performance and propose an alternative objective 
to align $p_S^\phi$ and $p_T^\phi$ instead of driving them to be identically distributed.

%% file: sections/motivation.tex
\section{A Motivating Scenario}
\label{sec:moti}

To motivate our approach, we formally show how 
exact distribution matching can lead to undesirable performance. 
More specifically, we will lower bound $\err_T(\phi, h)$  
when both $\err_S(\phi, h)$ and $D(p_S^\phi, p_T^\phi)$ are zero
with respect to the shift in the label distribution.
Let $\rho_S$ and $\rho_T$ be the proportion of data with positive label, i.e.,
$\rho_U = \int \dx p_U(x) f(x)$. 
We formalize the result as follows.
\begin{proposition}
If $D(p_S^\phi, p_T^\phi)=0$ if and only if $p_S^\phi \equiv p_T^\phi$,  $\err_S(\phi, h)=D(p_S^\phi, p_T^\phi)=0$ 
indicates $\err_T(\phi, h)\ge \abs{\rho_S-\rho_T}$.
\label{prop:limitation}
\end{proposition}
The proof follows the intuition of Figure~\ref{fig:toy-exact}: 
%
%
If $\rho_S<\rho_T$, the best we can do is to map $\rho_T-\rho_S$ 
proportion of positive samples from the target inputs
to regions of latent space corresponding 
to negative examples from the source domain
while maintaining the label consistency for remaining ones. 
Switching the term positive/negative gives a similar argument for $\rho_T<\rho_S$. 
Proposition~\ref{prop:limitation} says that 
if there is a label distribution mismatch $\rho_T\ne\rho_S$,  
minimizing the objective \eqref{eq:obj-general} to zero 
imposes a positive lower bound on the target error.
%
%
This is especially problematic in cases where 
a perfect pair $\phi, h$ may exist,
achieving zero error on both source and target data (Figure~\ref{fig:toy-relaxed}).

\textbf{Asymmetrically-relaxed distribution alignment } 
It may appear contradictory that minimizing 
the first and third term of \eqref{eq:tar-err-decomp} to zero 
guarantees a positive $\err_T(\phi, h)$
and thus a positive second term when there exists a pair of $\phi, h$ 
such that $\err_T(\phi, h)=0$ (all three terms are zero). 
%
%
However, this happens because although $D(p_S^\phi, p_T^\phi)=0$ 
is a sufficient condition for the third term of \eqref{eq:tar-err-decomp} to be zero, 
\emph{it is not a necessary condition}. 
We now examine
the third term of \eqref{eq:tar-err-decomp}:
\begin{align*}
& \int \dz \left(p_T^\phi(z) - p_S^\phi(z)\right) \risk_S(z; \phi, h) \\
& \le \left( \sup_{z\in \lsp}\frac{p_T^\phi(z)}{p_S^\phi(z)} - 1  \right) \err_S(\phi, h). \addeq\label{eq:term3-revisit}
\end{align*}
This expression \eqref{eq:term3-revisit} shows that
if the source error $\err_S(\phi, h)$ is zero 
then it is sufficient to say the third term of \eqref{eq:tar-err-decomp} is zero
when the density ratio $p_T^\phi(z)/p_S^\phi(z)$ 
is upper bounded by some constant for all $z$.
%
%
Note that it is impossible to bound $p_T^\phi(z)/p_S^\phi(z)$ 
by a constant that is smaller than $1$ 
so we write this condition as 
$\sup_{z\in \lsp}p_T^\phi(z)/p_S^\phi(z)\le 1+\beta$ 
for some $\beta \ge 0$. 
Note that this is a relaxed condition compared with $p_T^\phi(z)\equiv p_S^\phi(z)$, which is a special case with $\beta=0$.

%
%
Relaxing the exact matching condition to the more forgiving
bounded density ratio condition
makes it possible to obtain a perfect target domain classifier 
in many cases where the stricter condition does not,
by requiring only that the (latent space) target domain support 
is contained in the source domain support, 
as shown in Figure~\ref{fig:toy-relaxed}. 
The following proposition states that our relaxed matching condition
does not suffer from the previously-described problems 
concerning shifting label distributions (Proposition~\ref{prop:limitation}), 
and provides intuition regarding just how large $\beta$ 
may need to be to admit a perfect target domain classifier.

\begin{proposition}
For every $\rho_S, \rho_T$, there exists a construction of $(p_S, p_T, \phi, h)$ such that $\err_S(\phi, h)=0$, $\err_T(\phi, h)=0$ and $\sup_{z\in \lsp}p_T^\phi(z)/p_S^\phi(z)\le \max\left\{\frac{\rho_T}{\rho_S}, \frac{1-\rho_T}{1-\rho_S}\right\}$.
\label{prop:no-limitation}
\end{proposition}
Given this motivation,
we propose relaxing from exact distribution matching 
to bounding the density ratio in the domain-adversarial learning objective \eqref{eq:obj-general}. 
We call this \emph{asymmetrically-relaxed distribution alignment} 
since we aim at upper bounding $p_T^\phi/p_S^\phi$ 
(but not $p_S^\phi/p_T^\phi$).
We now introduce a class of distances between distributions 
that can be minimized to achieve the relaxed alignment:

\begin{definition}[$\beta$-admissible distances]
Given a family of distributions defined on the same space $\lsp$, 
a distance metric $D_\beta$ between distributions is called \emph{$\beta$-admissible} 
if $D_\beta(p,q)=0$ when $\sup_{z\in \lsp}p(z)/q(z)\le 1+\beta$ 
and $D_\beta(p,q)>0$ otherwise. 
\label{def:beta-dist}
\end{definition}

\textbf{Our proposed approach } is to \emph{replace the typical distribution distance $D$ in the domain-adversarial objective \eqref{eq:obj-general} 
with a $\beta$-admissible distance $D_\beta$} 
so that minimizing the new objective does not necessarily lead 
to a failure under label distribution shift. 
However, it is still premature to claim the justification 
of our approach due to the following issues:
\begin{enumerate*}[label=(\roman*)]
\item We may not be able get a perfect source domain classifier 
with $\err_S(\phi, h)=0$. 
This also indicates a trade-off in selecting $\beta$ as 
(a) higher $\beta$ will increase the upper bound ($\beta \err_S(\phi, h)$ 
according to \eqref{eq:term3-revisit}) on the third term in \eqref{eq:tar-err-decomp} 
(b) lower $\beta$ will make a good target classifier 
impossible under label distribution shift.
\item Minimizing $D_\beta(p_T^\phi, p_S^\phi)$ as part of an objective 
does not necessarily mean that we will obtain a solution 
with $D_\beta(p_T^\phi, p_S^\phi)=0$. 
There may still be some proportion of samples from the target domain 
lying outside the support of source domain in the latent space $\lsp$. 
In this case, the density ratio $p_T^\phi/p_S^\phi$ is unbounded 
and \eqref{eq:term3-revisit} becomes vacuous.
\item Even when we are able optimize the objective perfectly, 
i.e., $\err_S(\phi, h)=D_\beta(p_S^\phi, p_T^\phi)=0$,
with a proper choice of $\beta$ such that there exists $\phi, h$ 
such that $\err_T(\phi, h)=0$ holds simultaneously (e.g. Figure~\ref{fig:toy-relaxed},
Proposition~\ref{prop:no-limitation}), 
it is still not guaranteed that such $\phi, h$ is learned (e.g. Figure~\ref{fig:toy-fail}),
as the second term of \eqref{eq:tar-err-decomp} is unbounded and changes with $\phi$. 
%
%
Put simply, the problem is that although 
there may exist alignments perfect for prediction,
there also exist other alignments that satisfy the objective 
but predict poorly (on target data).
To our knowledge this problem effects all 
domain-adversarial methods proposed in the literature, 
and how to theoretically guarantee that the desired alignment is learned 
remains an open question.
\end{enumerate*}

Next, 
we theoretically study
the target classification error under 
asymmetrically-relaxed distribution alignment. 
Our analysis resolves the above issues by 
(i) working with imperfect source domain classifier 
and relaxed distribution alignment;
and (ii) providing concrete assumptions
under which a good target domain classifier 
can be learned. 

%% file: sections/theory.tex
\section{Bounding the Target Domain Error}
\label{sec:theory}

In a manner similar to \eqref{eq:tar-err-decomp}, \citet{ben2007analysis, ben2010theory} 
bound the target domain error by a sum of three terms: 
(i) the source domain error 
(ii) an $\HH$-divergence between $p_S^\phi$ and $p_T^\phi$ 
(iii) the best possible classification error 
that can be achieved on the combination of $p_S^\phi$ and $p_T^\phi$. 
We motivate our analysis by explaining why their results 
are insufficient to give a meaningful bound 
for domain-adversarial learning approaches. 
From a theoretical upper bound, we may desire to make claims in the following pattern: 

\emph{Let $\M_\A$ be a set of models that satisfy a set of properties $\A$ (e.g. with low training error), and $\B$ be a set of assumptions on the data distributions $(p_S, p_T, f)$. For any given model $M\in \M_\A$, its performance can be bounded by a certain quantity, i.e. $\err_T(M)\le \epsilon_{\A, \B}$.}

Ideally, $\A$ should be \emph{observable} on available data information (i.e. without knowing target labels), and assumptions $\B$ should be \emph{model-independent} (independent of which model $M=(\phi,h)$ is learned among $\M_\A$). 
In the results of \citet{ben2007analysis, ben2010theory}, terms (i) and (ii) are observable so $\A$ can be set as achieving low quantities on these two terms. 
Since term (iii) is unobservable we may want to make assumptions on it. 
This term, however, is model-dependent 
when $\phi$ is learned jointly. 
To make 
a \emph{model-independent} assumption on term (iii), 
we need to take the supremum over all $(\phi,h)\in \M_\A$, i.e., 
all possible models that achieve low values on (i) and (ii).
This supremum can be vacuous without further assumptions 
as a cross-label mapping may also achieve low source error 
and distribution alignment (e.g. Figure~\ref{fig:toy-fail} v.s. Figure~\ref{fig:toy-relaxed}). 
Moreover, when $\HH$ contains all possible binary classifiers, 
the $\HH$-divergence is minimized 
only if the two distributions are the same, 
thus suffering the same problem as Proposition~\ref{prop:limitation} 
and is therefore not suitable for motivating a learning objective.

To overcome these limitations, 
we propose a new theoretical bound on the target domain error which 
(a) treats the difference between $p_S^\phi$ and $p_T^\phi$ asymmetrically 
and
(b) bounds the label consistency (second term in \ref{eq:tar-err-decomp}) 
by exploiting the Lipschitz-ness of 
$\phi$
as well as the separation and connectedness of data distributions. 
%
%
Our result can be interpreted as a combination of \emph{observable model properties} and unobservable \emph{model-independent assumptions} while being non-vacuous: 
it is able to guarantee correct classification 
for (some fraction of) data points from the target domain 
even where the source domain has zero density.

\subsection{A general bound}

We introduce our result with the following construction: 
\begin{construction} The following statements hold simultaneously: 
\begin{enumerate} \setlength\itemsep{0em}
\item (\emph{Lipschitzness of representation mapping.}) $\phi$ is $L$-Lipschitz: $d_\lsp(\phi(x_1), \phi(x_2)) \le L d_\rsp(x_1, x_2)$ for any $x_1, x_2 \in \rsp$. \label{cons:lip}
\item (\emph{Imperfect asymmetrically-relaxed distribution alignment.}) For some $\beta\ge 0$, there exist a set $B\subset \lsp$ such that $\frac{p_T^\phi(z)}{p_S^\phi(z)} \le 1+\beta$ holds for all $z\in B$ and $p_T^\phi(B)\ge 1-\delta_1$. \label{cons:align}
\item (\emph{Separation of source domain in the latent space.}) There exist two sets $C_0, C_1 \subset \rsp$ that satisfy:\label{cons:sep}
\begin{enumerate}\setlength\itemsep{0em}
\item
$C_0\cap C_1 = \emptyset$
\item $p_S(C_0\cup C_1)\ge 1-\delta_2$. 
\item For $i\in \{0, 1\}$, $f(x)=i$ for all $x\in C_i$. \item $\inf_{z_0\in \phi(C_0), z_1\in \phi(C_1)} d_\lsp(z_0, z_1) \ge \Delta > 0$. 
\end{enumerate} 
\end{enumerate}
\label{cons:main}
\end{construction}
Note that this construction does not require any information about target domain labels 
so the statements~[\ref{cons:lip}-\ref{cons:sep}] can be viewed as \emph{observable properties} of  $\phi$. 
We now introduce our \emph{model-independent} assumption:
\begin{assumption} (\emph{Connectedness from target domain to source domain.})
Given constants $(L, \beta, \Delta, \delta_1, \delta_2, \delta_3)$, assume that, for any $B_S, B_T\subset \rsp$ with $p_S(B_S)\ge 1-\delta_2$ and $p_T(B_T)\ge 1-\delta_1-(1+\beta)\delta_2$, there exists $C_T\subset B_T$ that satisfies the following conditions:
\begin{enumerate}\setlength\itemsep{0em}
\item For any $x\in C_T$, there exists 
$x'\in C_T \cap B_S$ such that one can find a sequence of points $x_0,x_1,...,x_m \in C_T$ with $x_0=x$, $x_m=x'$, $f(x)=f(x')$ and $d_\rsp(x_{i-1}, x_i)< \frac{\Delta}{L}$ for all $i=1,..., m$.
\item $p_T(C_T) \ge 1- \delta_3$.
\end{enumerate}
\label{assumption:main}
\end{assumption}
We are ready to present our main result: 
\begin{theorem}
Given a $L$-Lipschitz mapping $\phi \in \fsp$ and a binary classifier $h\in \HH$, 
if $\phi$ satisfies the properties in Construction~\ref{cons:main} 
with constants $(L, \beta, \Delta, \delta_1, \delta_2)$, and Assumption~\ref{assumption:main} holds with the same set of constants plus $\delta_3$,
then the target domain error can be bounded as
\begin{align*}
\err_T(\phi, h) \le (1+\beta) \err_S(\phi, h) + 3\delta_1 + 2(1+\beta)\delta_2 + \delta_3 \,. 
\end{align*}
\label{thm:main}
\end{theorem}
Notice that 
it is always possible to make Construction~\ref{cons:main} 
by adjusting the constants $L, \beta, \Delta, \delta_1, \delta_2$. Given these constants,
Assumption~\ref{assumption:main} can always be satisfied by adjusting $\delta_3$. So Theorem~\ref{thm:main} is a general bound.
 
The key challenge in 
bounding $\err_T(\phi, h)$
is to bound the second term in \eqref{eq:tar-err-decomp}
by identifying sufficient conditions that prevent cross-label mapping (e.g. Figure~\ref{fig:toy-fail}).  
To resolve this challenge, we exploit the fact that 
if there exist a path from a target domain sample to a source domain sample 
in the input space $\rsp$ and all samples along the path 
are mapped into two separate regions in the latent space 
(due to distribution alignment), 
then these two connected samples cannot be mapped to different regions,
as shown in Figure~\ref{fig:toy-bound}.

\begin{figure}[t]
\centering
\subfigure[Failure case]{
\includegraphics[width=0.45\columnwidth]{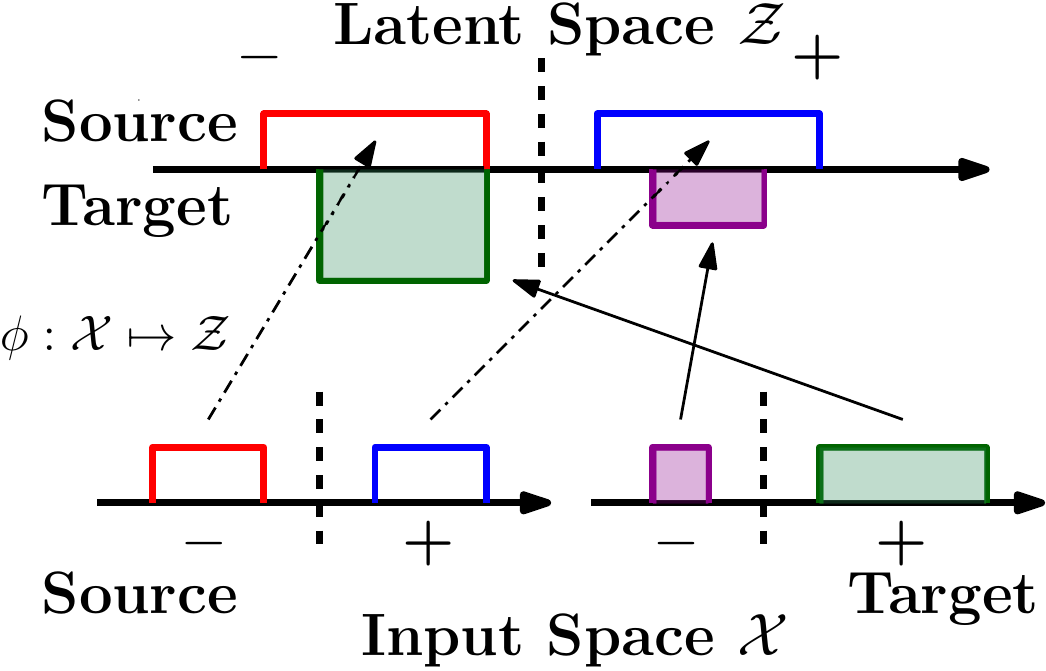}
\label{fig:toy-fail}
}
\subfigure[Failure impossible]{
\includegraphics[width=0.45\columnwidth]{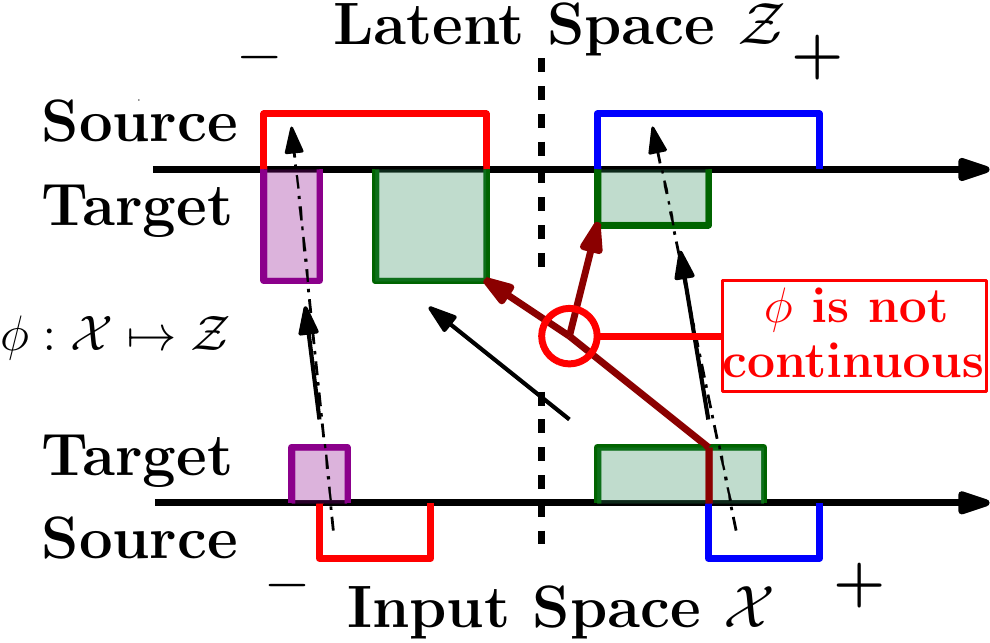}
\label{fig:toy-bound}
}
\caption{
(a) Label consistency is broken even if $\phi$ satisfies the relaxed distribution aligning requirement. (b) The main idea of our analysis: A continuous mapping cannot project a connected region into two regions separated by a margin. So label consistency is preserved for a region that is connected to the source domain.}
\label{fig:toy-theory}
\end{figure}


\subsection{Example of a perfect target domain classifier}

To interpret our result, we construct a simple situation where $\err_T (\phi, h)=0$ 
is guaranteed when the domain adversarial objective with relaxed distribution alignment is minimized to zero, exploiting pure data-dependent assumptions:
\begin{assumption}
Assume the target support consists of disjoint clusters 
$\supp{p_T}=S_{T,0,1}\cup ...\cup S_{T,0,m_0}\cup S_{T,1,1}\cup ...\cup S_{T, 1, m_1}$, 
where any cluster $S_{T,i,j}$ is connected 
and its labels are consistent: 
$f(x) = i$ for all $x\in S_{T,i,j}$.
Moreover, each of these cluster overlaps with source distribution. 
That is, for any $i\in \{0, 1\}$ and $j\in \{1,...,m_i\}$, $S_{T,i,j}\cap \supp{p_S} \ne \emptyset$.
\label{assumption:simp}
\end{assumption}
\begin{corollary}
If Assumption~\ref{assumption:simp} holds 
and there exists a continuous mapping $\phi$ 
such that (i) $\sup_{z\in \lsp} p_T^\phi(z)/p_S^\phi(z)\le 1+\beta$ for some $\beta\ge 0$; 
(ii) 
for any pair $x_0, x_1\in \supp{p_S}$ such that 
$f(x_0)=0$ and $f(x_1)=1$, we have $d_{\lsp} (\phi(x_0), \phi(x_1)) \ge  \Delta>0$, 
then $\err_S(\phi, h)=0$ indicates $\err_T(\phi, h)=0$. 
\label{coro:simp}
\end{corollary}
Proof follows directly by observing that a construction of  $\delta_1=\delta_2=\delta_3=0$ exists in Theorem~\ref{thm:main}. A simple example that satisfies Assumption~\ref{assumption:simp} is Figure~\ref{fig:toy-bound}.
For a real world example, consider the cat-dog classification problem. 
Say that source domain contains small-to-medium cats and dogs while target domain contains medium-to-large cats and dogs. 
The target domain consists of clusters (e.g. cats and dogs, or multiple sub-categories) 
and each of them overlaps with the source domain (the medium ones).

%% file: sections/algorithms.tex
\section{Asymmetrically-relaxed distances 
}
\label{sec:dists}

So far, we have motivated the use of asymmetrically-relaxed distribution alignment 
which aims at bounding $p_T^\phi/p_S^\phi$ by a constant 
instead of driving towards $p_S^\phi\equiv p_T^\phi$. 
More specifically, we propose to use a \emph{$\beta$-admissible} 
(Definition~\ref{def:beta-dist}) distance $D_\beta$ in objective \eqref{eq:obj-general} 
to align the source and target encodings 
rather than the standard distances corresponding an adversarial domain classifier. 
In this section, we derive several \emph{$\beta$-admissible} distance metrics  
that can be practically minimized with adversarial training.
More specifically, we propose three types of distances 
(i) f-divergences; (ii) modified Wasserstein distance; (iii) reweighting distances;
and demonstrate how to optimize them by adversarial training. 

\subsection{$f$-divergence}

Given a convex and continuous function $f$ which satisfies $f(1)=0$, 
the $f$-divergence between two distributions $p$ and $q$ 
can be written as $D_f(p, q) = \int \dz p(z) f\left(\frac{q(z)}{p(z)}\right)$.
According to Jensen's inequality $D_f(p, q)\ge  f\left(\int \dz  p(z)\frac{q(z)}{p(z)}\right)=0$. 
Standard choices of $f$ (see a list in \citet{nowozin2016f}) 
are strictly convex thus $D_f(p,q)=0$ if and only if $p\equiv q$ when $f$ is strictly convex. 
To derive a $\beta$-adimissible variation for each standard choice of $f$,
we linearize $f(u)$ where $u\ge \frac{1}{1+\beta}$.
If and only if $\frac{p(z)}{q(z)}\le 1+\beta$ for all $z$, $f$ 
becomes a linear function with respect to all $q(z)/p(z)$ 
and thus Jensen's inequality holds with equality.  

Given a convex, continuous function $f:\RR^+\mapsto \RR$ 
with $f(1)=0$ and some $\beta\ge 0$, 
we introduce the partially linearized $\bar{f}_\beta$ as follows
\begin{align*}
\bar{f}_\beta(u) = 
\begin{cases} 
f(u) + C_{f,\beta}  & \text{ if } u\le \frac{1}{1+\beta} \,,\\
f'(\frac{1}{1+\beta})u - f'(\frac{1}{1+\beta}) & \text{ if } u> \frac{1}{1+\beta} \,.
\end{cases}
\end{align*}
where $C_{f,\beta} = - f(\frac{1}{1+\beta}) + f'(\frac{1}{1+\beta})\frac{1}{1+\beta} - f'(\frac{1}{1+\beta})$. 

It can be shown that $\bar{f}_\beta$ is continuous, convex and $\bar{f}_\beta(1)=0$. 
As we already explained, $D_{\bar{f}_\beta}(p,q)=0$ 
if and only if $\frac{p(z)}{q(z)}\le 1+\beta$ 
for all $z$. Hence is $D_{\bar{f}_\beta}$ is $\beta$-admissible.

\textbf{Adversarial training } 
According to \citet{nowozin2016f}, adversarial training \citep{goodfellow2014generative} can be viewed as minimizing the dual form of $f$-divergences
\begin{align*}
D_f(p, q) = \sup_{T:\lsp \mapsto \dom{f^*}} \EE{z\sim q}{T(z)} - \EE{z\sim p}{f^*(T(z))}
\end{align*}
where $f^*$ is the Fenchel Dual of $f$ with $f^*(t)=\sup_{u\in \dom{f}} \left\{ ut - f(u) \right\}$. Applying the same derivation for $\bar{f}_\beta$ we get\footnote{We are omitting some additive constant term.}
\begin{align*}
& D_{\bar{f}_\beta}(p, q) = \sup_{T:\lsp \mapsto \dom{\bar{f}_\beta^*}} \EE{z\sim q}{T(z)} - \EE{z\sim p}{f^*(T(z))} \addeq\label{eq:f-dual}
\end{align*}
where $\dom{\bar{f}_\beta^*}=\dom{f^*} \cap \big(-\infty, f'(\frac{1}{1+\beta})\big]$.

Plugging in the corresponding $f$ for JS-divergence gives
\begin{align*}
& D_{\bar{f}_\beta}(p,q) 
\\ 
& = \sup_{g:\lsp \mapsto (0,1]} \EE{z\sim q}{\log\frac{g(z)}{2+\beta}} + \EE{z\sim p}{\log\left(1- \frac{g(z)}{2+\beta}\right)} \addeq\label{eq:f-dual-gan}\,,
\end{align*}
where $g(z)$ can be parameterized by a neural network with sigmoid output as typically used in adversarial training. 

\subsection{Wasserstein distance}

The idea behind modifying the Wasserstein distance 
is to model the optimal transport from $p$ 
to the region where distributions have 
$1+\beta$ maximal density ratio with respect to $q$. 
We define the relaxed Wassertein distance as
\begin{align*}
W_\beta (p, q) = \inf_{\gamma \in \prod_\beta(p, q)} \EE{(z_1, z_2)\sim \gamma}{\norm{z_1-z_2}} \,,
\end{align*}
where $\prod_\beta(p, q)$ is defined as the set of joint distributions $\gamma$ over $\lsp \times \lsp$ such that
\begin{align*}
& \forall z_1 \int \dz \gamma(z_1, z) = p(z_1) \,; \forall z_2 \int \dz \gamma(z, z_2) \le (1+\beta) q(z_2) \,.
\end{align*}
$W_\beta$ is $\beta$-admissible since no transportation is needed 
if $p$ already lies in the qualified region with respect to $q$.

\textbf{Adversarial training } 
Following the derivation for the original Wasserstein distance, the dual form becomes 
\begin{align*}
W_\beta (p, q) & = \sup_{g} \EE{z \sim p}{g(z)} - (1+\beta)\EE{z \sim q}{g(z)} \addeq \label{eq:w-dist}\\
\text{ s.t. } & \forall z\in \lsp \,, g(z) \ge 0 \,, \\
&  \forall z_1, z_2 \in \lsp \,,  g(z_1) - g(z_2) \le \norm{z_1-z_2} \,, 
\end{align*}
Optimization with adversarial training can be done by parameterizing $g$ 
as a non-negative function (e.g. with soft-plus output $\log (1+e^x)$ or RELU output $\max(0, x)$) and following
\citet{arjovsky2017wasserstein, gulrajani2017improved} 
to enforce its Lipschitz continuity approximately.

\subsection{Reweighting distance}
Given any distance metric $D$, a generic way to make it $\beta$-admissible 
is to allow reweighting for one of the distances within a $\beta$-dependent range. 
The relaxed distance is then defined as the minimum achievable distance by such reweighting.

Given a distribution $q$ over $\lsp$ and a reweighting function $w: \lsp\mapsto [0, \infty)$. 
The reweighted distribution $q_w$ is defined as $q_{w}(z) = \frac{q(z)w(z)}{\int \dz q(z)w(z)}$. 
Define $\W_{\beta, q}$ to be a set of \emph{$\beta$-qualified} reweighting with respect to $q$:
\begin{align*}
\W_{\beta, q} = \left\{w: \lsp \mapsto [0, 1], 
\int \dz q(z)w(z) = \frac{1}{1+\beta} \right\}\,.
\end{align*}
Then the relaxed distance can be defined as
\begin{align*}
D_\beta(p, q) = \min_{w\in \W_{\beta, q}} D(p, q_w) \,. \addeq\label{eq:dist-w}
\end{align*}
Such $D_\beta$ is $\beta$-admissible since the set $\{q_w:w\in \W_{\beta, q}\}$ is exactly the set of $p$ such that $\sup_{z\in \lsp} p(z)/q(z)\le 1+\beta$.

\textbf{Adversarial training } 
We propose an \emph{implicit-reweighting-by-sorting} approach to optimize $D_\beta$ 
without parameterizing the function $w$ when $D$ can be optimized by adversarial training. 
Adversarially trainable $D$ shares a general form as
\begin{align*}
D(p,q) = \sup_{g\in \mathcal{G}} \EE{z\sim p}{f_1(g(z))} - \EE{z\sim q}{f_2(g(z))} \,,
\end{align*}
where $f_1$ and $f_2$ are monotonically increasing functions. 
According to \eqref{eq:dist-w}, the relaxed distance can be written as 
\begin{align*}
& D_\beta(p,q) = \min_w \sup_{g\in \mathcal{G}} \EE{z\sim p}{f_1(g(z))} - \EE{z\sim q_w}{f_2(g(z))} \,, \\ &
\text{ s.t. } w: \lsp \mapsto [0,1] \,, 
\int \dz q(z)w(z) = \frac{1}{1+\beta} \,. \addeq\label{eq:dist-w-obj}
\end{align*}
One step of alternating minimization 
on $D_\beta$, 
could consist of fixing $p, q, g$ and optimizing $w$. 
Then the problem becomes 
\begin{align*}
\max_{w\in \W_{\beta, q}} 
\int \dz q(z) w(z) f_2(g(z)) \addeq\label{eq:dist-w-sort} \,.
\end{align*}
Observe that the optimal solution to \eqref{eq:dist-w-sort} 
is to assign $w(z)=1$ for the $\frac{1}{1+\beta}$ fraction of $z$ from distribution $q$,
where $f_2(g(z))$ take the largest values. 
Based on this observation, we propose to do the following sub-steps 
when optimizing \eqref{eq:dist-w-sort} as an alternating minimization step: 
(i) Sample a minibatch of $z\sim q$;
(ii) Sort these $z$ in descending order according to $f_2(g(z))$;
(iii) Assign $w(z)=1$ to the first $\frac{1}{1+\beta}$ fraction of the list. 
Note that this optimization procedure is not justified in principle 
with mini-batch adversarial training but we found it to work well in our experiments.

%% file: sections/experiments.tex
\begin{figure*}[h]
\centering
\subfigure[raw (synthetic) data]{
\includegraphics[width=0.31\linewidth]{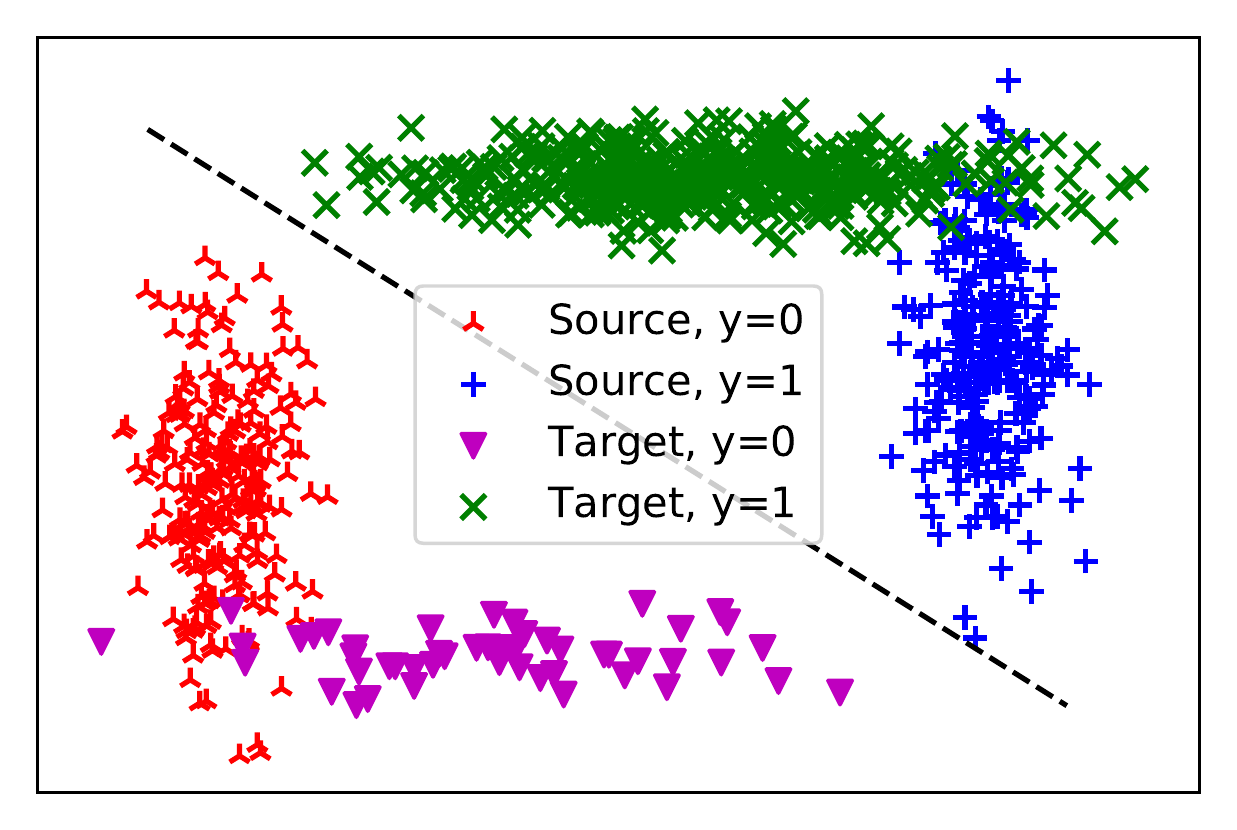}
\label{fig:mixgs-data}
}
\subfigure[latent representations (DANN)]{
\includegraphics[width=0.31\linewidth]{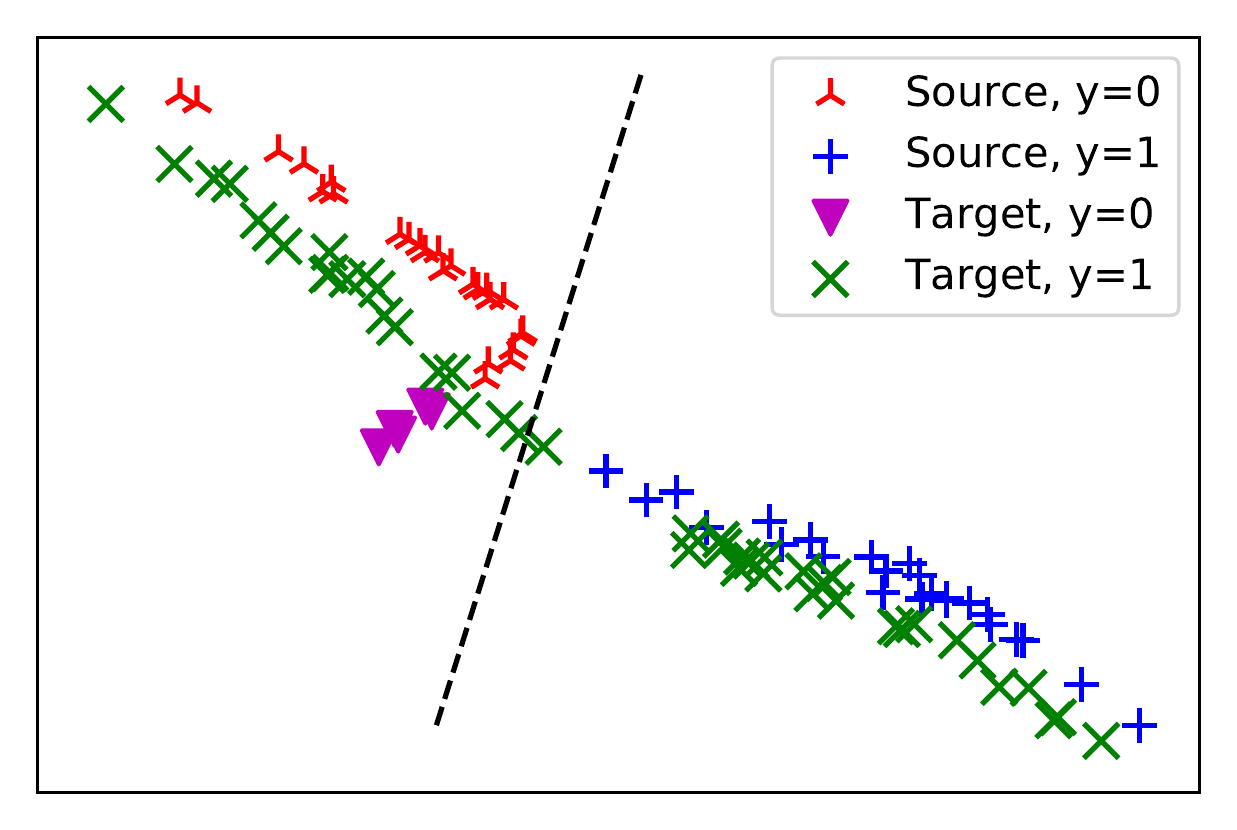}
\label{fig:mixgs-fail}
}
\subfigure[latent representations (ours)]{
\includegraphics[width=0.31\linewidth]{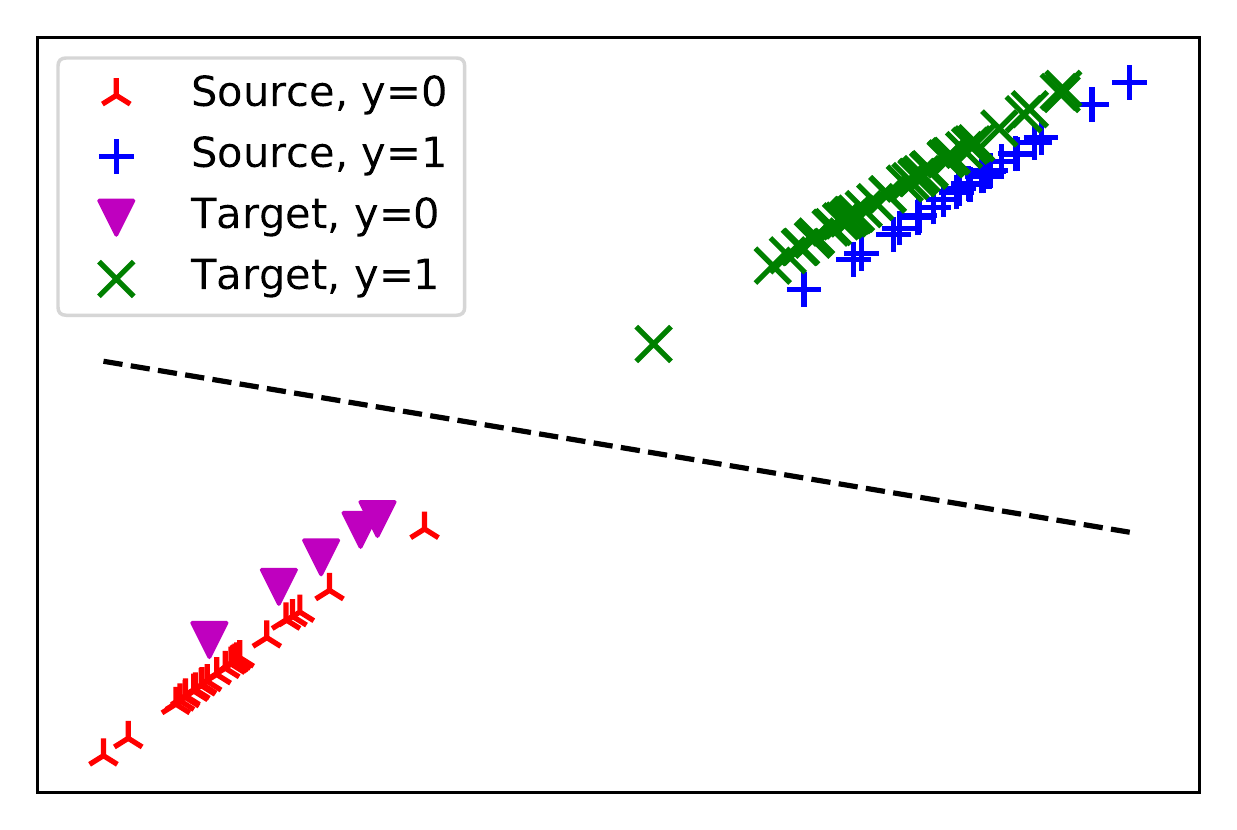}
\label{fig:mixgs-latent}
}

\caption{
Domain-adversarial training under label distribution shift on a synthetic dataset. 
}

\label{fig:mixgs-intro}
\end{figure*}

\section{Experiments}

To evaluate our approach, 
we implement Domain Adversarial Neural Networks (DANN), \citep{ganin2016domain} 
replacing the JS-divergence (domain classifier) 
with our proposed $\beta$-admissible distances (Section~\ref{sec:dists}). 
Our experiments address the following questions: 
(i) Does DANN suffer the limitation as anticipated (Section~\ref{sec:moti})
when faced with label distribution shift? 
(ii) If so, do our $\beta$-admissible distances overcome these limitations? 
(iii) Absent shifting label distributions, is our approach comparable to DANN?

We implement adversarial training with different $\beta$-admissible distances 
(Section~\ref{sec:dists}) 
and compare their performance with vanilla DANN. 
We name different implementations as follows. \begin{enumerate*}[label=(\alph*)]
\item {\sc Source}: source-only training. 
\item {\sc DANN}: JS-divergence (original DANN). 
\item {\sc WDANN}: original Wasserstein distance. 
\item {\sc fDANN-$\beta$}: $\beta$-admissible $f$-divergence, JS-version \eqref{eq:f-dual-gan}. 
\item {\sc sDANN-$\beta$}: reweighting JS-divergence \eqref{eq:dist-w-obj}, optimized by our proposed \emph{implicit-reweighting-by-sorting}.
\item {\sc WDANN1-$\beta$}: $\beta$-admissible Wasserstein distance \eqref{eq:w-dist} with soft-plus on critic output. 
\item {\sc WDANN2-$\beta$}: $\beta$-admissible Wasserstein distance \eqref{eq:w-dist} with RELU on critic output. 
\item {\sc sWDANN-$\beta$}: reweighting Wasserstein distance \eqref{eq:dist-w-obj}, 
optimized by \emph{implicit-reweighting-by-sorting}.\end{enumerate*} 
Adversarial training on Wasserstein distances follows \citet{gulrajani2017improved} 
but uses one-sided gradient-penalty. 
We always perform adversarial training with alternating minimization (see Appendix for details).

\paragraph{Synthetic datasets } 
We create a mixture-of-Gaussians binary classification dataset 
where each domain contains two Gaussian distributions, one per label. 
For each label, the distributions in source and target domain have a small overlap, 
validating the assumptions in our analysis. 
We create a label distribution shift with balanced source data (50\% 0's v.s. 50\% 1's) 
and imbalanced target data (10\% 0's v.s. 90\% 1's) 
as shown in Figure~\ref{fig:mixgs-data}. 
Table~\ref{tb:syn} shows the target domain accuracy for different approaches. 
As expected, vanilla DANN fails under label distribution shift 
because a proportion of samples from the target inputs 
are mapped to regions of latent space corresponding to negative samples from the source domain  (Figure~\ref{fig:mixgs-fail}). 
In contrast, with our $\beta$-admissible distances, 
domain-adversarial networks are able to adapt successfully (Figure~\ref{fig:mixgs-latent}), 
improving target accuracy from 89\% (source-only) to 99\% accuracy (with adaptation), 
except the cases where $\beta$ is too small to admit a good target domain classifier 
(in this case we need $\beta \ge 0.9/0.5 - 1 =0.8$). 
We also experiment with label-balanced target data (no label distribution shift). 
All approaches except source-only achieve an accuracy above 99\%, 
so we do not present these results in a separate table.

\begin{table}[ht]
\caption{Classification accuracy on target domain 
with label distribution shift on a synthetic dataset.}
\label{tb:syn}
\begin{center}
\begin{small}
\begin{sc}
\begin{tabular}{lcccr}
\toprule
method & accuracy\% \\
\midrule
Source & 89.4$\pm$1.1 \\
DANN & 59.1$\pm$5.1 &
WDANN & 50.8$\pm$32.1 \\
\midrule
$\beta$ & 0.5 & 2.0 & 4.0 \\
\midrule
fDANN-$\beta$   & 66.0$\pm$ 41.6& \textbf{99.9$\pm$ 0.0}& 99.8$\pm$0.0\\
sDANN-$\beta$  & \textbf{99.9$\pm$ 0.1}& \textbf{99.9$\pm$ 0.0}& \textbf{99.9$\pm$0.0}\\
WDANN1-$\beta$ & 45.7$\pm$ 41.5& 66.4$\pm$ 41.1& \textbf{99.9$\pm$0.0}\\
WDANN2-$\beta$ & 97.6$\pm$ 1.2& 99.7$\pm$ 0.2& 99.5$\pm$0.3\\
sWDANN-$\beta$ & 79.0$\pm$ 5.9& \textbf{99.9$\pm$ 0.0}& \textbf{99.9$\pm$0.0}\\
\bottomrule
\end{tabular}
\end{sc}
\end{small}
\end{center}
\end{table}

\paragraph{Real datasets}
We experiment with the MNIST and USPS handwritten-digit datasets.
For both directions (MNIST $\rightarrow$ USPS and USPS $\rightarrow$ MNIST), 
we experiment both with and without label distribution shift. 
The source domain is always class-balanced. 
To simulate label distribution shift, 
we sample target data from only half of the digits, 
e.g. [0-4] or [5-9]. 
Tables~\ref{tb:mnist-usps} and \ref{tb:usps-mnist} 
show the target domain accuracy for different approaches 
with/without label distribution shift. 
As on synthetic datasets, 
we observe that DANN performs much worse than source-only training 
under label distribution shift. 
Compared to the original DANN, our approaches fair significantly better 
while achieving comparable performance absent label distribution shift. 

\begin{table}[ht]
\caption{Classification accuracy on target domain with/without label distribution shift on MNIST-USPS.}
\label{tb:mnist-usps}
\vskip 0.15in
\begin{center}
\begin{small}
\begin{sc}
\begin{tabular}{lcccr}
\toprule
target & [0-4] & [5-9] & [0-9] \\
labels & Shift & Shift & No-Shift 
\\
\midrule
Source      & 74.3$\pm$1.0  & 59.5$\pm$3.0  & 66.7$\pm$2.1 \\
DANN        & 50.0$\pm$1.9  & 28.2$\pm$2.8  & 78.5$\pm$1.6 \\
\midrule
fDANN-$1$   & 71.6$\pm$4.0  & \textbf{67.5$\pm$2.3}  & 73.7$\pm$1.5\\
fDANN-$2$   & 74.3$\pm$2.5  & 61.9$\pm$2.9  & 72.6$\pm$0.9\\
fDANN-$4$   & 75.9$\pm$1.6  & 64.4$\pm$3.6  & 72.3$\pm$1.2\\
sDANN-$1$   & 71.6$\pm$3.7  & 49.1$\pm$6.3  & 81.0$\pm$1.3\\
sDANN-$2$   & 76.4$\pm$3.1  & 48.7$\pm$9.0  & 81.7$\pm$1.4\\
sDANN-$4$   & \textbf{81.0$\pm$1.6}  & 60.8$\pm$7.5  & \textbf{82.0$\pm$0.4}\\
\bottomrule
\end{tabular}
\end{sc}
\end{small}
\end{center}
\end{table}

\begin{table}[ht]
\caption{Classification accuracy on target domain with/without label distribution shift on USPS-MNIST.}
\label{tb:usps-mnist}
\begin{center}
\begin{small}
\begin{sc}
\begin{tabular}{lcccr}
\toprule
target & [0-4] & [5-9] & [0-9] \\
labels & Shift & Shift & No-Shift \\
\midrule
Source      & 69.4$\pm$2.3  & 30.3$\pm$2.8  & 49.4$\pm$2.1 \\
DANN        & 57.6$\pm$1.1  & 37.1$\pm$3.5  & \textbf{81.9$\pm$6.7} \\
\midrule
fDANN-$1$   & 80.4$\pm$2.0  & 40.1$\pm$3.2  & 75.4$\pm$4.5\\
fDANN-$2$   & \textbf{86.6$\pm$4.9}  & 41.7$\pm$6.6  & 70.0$\pm$3.3\\
fDANN-$4$   & 77.6$\pm$6.8  & 34.7$\pm$7.1  & 58.5$\pm$2.2\\
sDANN-$1$   & 68.2$\pm$2.7  & \textbf{45.4$\pm$7.1}  & 78.8$\pm$5.3\\
sDANN-$2$   & 78.6$\pm$3.6  & 36.1$\pm$5.2  & 77.4$\pm$5.7\\
sDANN-$4$   & 83.5$\pm$2.7  & 41.1$\pm$6.6  & 75.6$\pm$6.9\\
\bottomrule
\end{tabular}
\end{sc}
\end{small}
\end{center}
\end{table}

%% file: sections/related.tex
\section{Related work}

Our paper makes distinct theoretical and algorithmic contributions
to the domain adaptation literature. 
Concerning theory, we provide a risk bound 
that explains the behavior of domain-adversarial methods 
with model-independent assumptions on data distributions. 
Existing theories without assumptions of contained support 
\citep{ben2007analysis, ben2010theory, ben2014domain, mansour2009domain, cortes2011domain} 
do not exhibit this property since 
(i) when applied to the input space, 
their results are not concerned with domain-adversarial learning as no latent space is introduced, 
(ii) when applied to the latent space, 
their unobservable constants/assumptions become $\phi$-dependent, 
which is undesirable as explained in Section~\ref{sec:theory}. 
Concerning algorithms, several prior works demonstrate empirical success of domain-adversarial approaches, 
\citep{tzeng2014deep, ganin2016domain, bousmalis2016domain, tzeng2017adversarial, hoffman2017cycada, shu2018dirt}. 
Among those, \citet{cao2018partiala, cao2018partialb} 
deal with the label distribution shift scenario through a heuristic reweighting scheme. 
However, their re-weighting presumes that they have a good classifier in the first place,
creating a cyclic dependency.

%% file: sections/conclusion.tex
\section{Conclusions}

We propose to use asymmetrically-relaxed distribution distances in domain-adversarial learning objectives, 
replacing standard ones which seek exact distribution matching in the latent space. 
While overcoming some limitations of the standard objectives under label distribution mismatch, 
we provide a theoretical guarantee for target domain performance under assumptions on data distributions. 
As our connectedness assumptions may not cover all cases where we expect domain adaptation to work in practice, 
(e.g. when the two domains are completely disjoint), 
providing analysis under other type of assumptions might be of future interest.

%% file: sections/appendix.tex
\section{Proofs}
\begin{proof}[Derivation of \eqref{eq:risk}]
\begin{align*}
\err_U(\phi, h)
& = \int \dx p_U(x)  \abs{h(\phi(x))-f(x)} \\
& = \int \dx \int \dz p_U^\phi(z)\phi_U(x|z)\abs{h(\phi(x))-f(x)} \\
& = \int \dz p_U^\phi(z) \int \dx \phi_U(x|z)\abs{h(z)-f(x)} \\
& = \int \dz p_U^\phi(z) \abs{h(z)-\int \dx \phi_U(x|z)f(x)} \\
& \doteq \int \dz p_U^\phi(z) \abs{h(z) - f^\phi_U(z)} \\
& \doteq \int \dz p_U^\phi(z) \risk_U(z; \phi, h)
\end{align*}
where we use the following fact: 
For any fixed $z$, $h(z)\in \{0, 1\}$, if $h(z)=0$ 
then $|h(z)-f(x)| = f(x)-h(z)$ for all $x$. 
Similarly, when $h(z)=1$, we have $|h(z)-f(x)| = h(z)-f(x)$ for all $x$. 
Thus we can move the integral over $x$ inside the absolute operation. 
\end{proof}

\begin{proof}[Proof of Proposition~\ref{prop:limitation}]
First we have 
\begin{align*}
\rho_U = \int \dx p_U(x) f(x) = \int \dx \int \dz p_U^\phi(z) \phi_U(x|z) f(x) = \int \dz p_U^\phi(z) f_U^\phi(z) \,.
\end{align*}
When $\err_S(\phi, h)=0$ we have 
\begin{align*}
\abs{\int \dz p_S^\phi(z) h(z)-\rho_S} = \abs{\int \dz p_S^\phi(z) h(z)-\int \dz p_S^\phi(z) f_S^\phi(z)} \le \int \dz p_S^\phi(z) \abs{h(z)-f_S^\phi(z)} = \err_S(\phi, h)=0
\end{align*}
thus $\int \dz p_S^\phi(z) h(z)=\rho_S$.

Applying the fact that $p_S^\phi(z) = p_T^\phi(z)$ for all $z \in \lsp$,
\begin{align*}
& \err_T(\phi, h)  = \int \dz p_T^\phi(z) \abs{h(z)-f_T^\phi(z)} \ge   \abs{\int \dz p_T^\phi(z)h(z)-\int \dz p_T^\phi(z)f_T^\phi(z)} \\
& = \abs{\int \dz p_S^\phi(z)h(z)-\int \dz p_T^\phi(z)f_T^\phi(z)} = \abs{\rho_S-\rho_T}\,,
\end{align*}
which concludes the proof.
\end{proof}

\begin{proof}[Proof of Proposition~\ref{prop:no-limitation}]
Let $p_S$ be the uniform distribution over $[0,1]$ and $p_T$ be the uniform distribution over $[2,3]$. The labeling function $f$ is set as $f(x)=1$ iff $x\in [0, \rho_S] \cup [2, 2+\rho_T]$ such that the definition of $\rho_S$ and $\rho_T$ is preserved. 
We construct the following mapping $\phi$: For $x\in [0,1]$ $\phi(x)=x$. For $x\in [2,2+\rho_T]$ $\phi(x)=(x-2)\rho_S/\rho_T$. For $x\in [2+\rho_T, 3]$ $\phi(x)=1 - (3-x)(1-\rho_S)/(1-\rho_T)$. 
$\phi$ maps both source and target data into $[0,1]$ with $p_S^\phi$ to be uniform over $[0,1]$ and  
$p_T^\phi(z)=\rho_T/\rho_S$ when $z\in [0, \rho_S]$ and 
$p_T^\phi(z)=(1-\rho_T)/(1-\rho_S)$ when $z\in [\rho_S, 1]$. 
Since $p_S^\phi(z)=1$ for all $z\in [0,1]$ we can conclude that 
$\sup_{z\in \lsp}p_T^\phi(z)/p_S^\phi(z)\le \max\left\{\frac{\rho_T}{\rho_S}, \frac{1-\rho_T}{1-\rho_S}\right\}$\,.
\end{proof}

\begin{proof}[Proof of Theorem~\ref{thm:main}]

Instead of working with Assumption~\ref{assumption:main} we first extend Construction~\ref{cons:main} with the following addition
\begin{construction}
(\emph{Connectedness from target domain to source domain.}) Let $C_T\subset \rsp$ be a set of points in the raw data space that satisfy the following conditions: \label{cons:conn}
\begin{enumerate}\setlength\itemsep{0em}
\item $\phi(C_T) \subset \phi(C_0 \cup C_1)$.
\item For any $x\in C_T$, there exists 
$x'\in C_T \cap (C_0 \cup C_1)$ such that one can find a sequence of points $x_0,x_1,...,x_m \in C_T$ with $x_0=x$, $x_m=x'$, $f(x)=f(x')$ and $d_\rsp(x_{i-1}, x_i)< \frac{\Delta}{L}$ for all $i=1,..., m$.
\item $p_T(C_T) \ge 1- \delta_3$.
\end{enumerate}
\end{construction}

We now proceed to prove bound based on Constructions~\ref{cons:main} and \ref{cons:conn}. Later on we will show that Assumption~\ref{assumption:main} indicates the existence of Construction~\ref{cons:conn} so that the bound holds with a combination of Constructions~\ref{cons:main} and Assumption~\ref{assumption:main}.

The third term of \eqref{eq:tar-err-decomp} can be written as 
\begin{align*}
& \int \dz p_S^\phi(z) \left(\frac{p_T^\phi(z)}{p_S^\phi(z)} - 1\right) \risk_S(z; \phi, h) \\
& \le \inf_{B\subseteq \lsp} \int_B \dz p_S^\phi(z) \left(\frac{p_T^\phi(z)}{p_S^\phi(z)} - 1\right) \risk_S(z; \phi, h)  + \int_{B^c} \dz p_S^\phi(z) \left(\frac{p_T^\phi(z)}{p_S^\phi(z)} - 1\right) \risk_S(z; \phi, h) \\
& \le \inf_{B\subseteq \lsp} \left( \sup_{z\in B}\frac{p_T^\phi(z)}{p_S^\phi(z)} - 1  \right)\int_B \dz p_S^\phi(z) \risk_S(z; \phi, h) + \int_{B^c} \dz p_T^\phi(z) \risk_S(z; \phi, h) \\
& \le \inf_{B\subseteq \lsp} \left( \sup_{z\in B}\frac{p_T^\phi(z)}{p_S^\phi(z)} - 1  \right)\err_S(\phi, h) + p_T^\phi(B^c) \\
& \le \beta\err_S(\phi, h) + \delta_1 \,. \addeq\label{eq:main-part1} 
\end{align*}

For the second term of \eqref{eq:tar-err-decomp}, plugging in $\risk_U(z; \phi, h) = \abs{h(z) - f^\phi_U(z)}$ gives
\begin{align*}
& \int \dz p_T^\phi(z) \left(\risk_T(z; \phi, h)- \risk_S(z; \phi, h)\right) \\ 
& = \int \dz p_T^\phi(z) \left(\abs{h(z) - f^\phi_T(z)}- \abs{h(z) - f^\phi_S(z)}\right) \\
& = \int \dz p_T^\phi(z) \abs{f^\phi_T(z)-f^\phi_S(z)} \\
& = \int \dz p_T^\phi(z) \abs{f^\phi_T(z)-f^\phi_S(z)} \left( \I{z\in \phi(C_0)} + \I{z\in \phi(C_1)} + \I{z\in \left(\phi(C_0)\cup\phi(C_1)\right)^c} \right) \\
& = \int \dz p_T^\phi(z) \abs{f^\phi_T(z)-f^\phi_S(z)} \I{z\in \phi(C_0)} + \int \dz p_T^\phi(z) \abs{f^\phi_T(z)-f^\phi_S(z)} \I{z\in \phi(C_1)} \\
& + \int \dz p_T^\phi(z) \abs{f^\phi_T(z)-f^\phi_S(z)} \I{z\in \left(\phi(C_0)\cup\phi(C_1)\right)^c}
\addeq\label{eq:gb-decomp1}
\end{align*}
Applying $\abs{f^\phi_T(z)-f^\phi_S(z)}\le f^\phi_T(z) + f^\phi_S(z)$ to the first part of \eqref{eq:gb-decomp1} gives
\begin{align*}
& \int \dz p_T^\phi(z) \abs{f^\phi_T(z)-f^\phi_S(z)} \I{z\in \phi(C_0)} \\
& \le \int \dz p_T^\phi(z) f^\phi_T(z)  \I{z\in \phi(C_0)}  +  \int \dz p_T^\phi(z)f^\phi_S(z) \I{z\in \phi(C_0)} \\
& = \int \dz p_T^\phi(z) \int \dx \phi_T(x|z)f(x)  \I{z\in \phi(C_0)}  +  \int \dz p_T^\phi(z)f^\phi_S(z) \I{z\in \phi(C_0)} \\
& = \int \dx f(x) \int \dz p_T^\phi(z)  \phi_T(x|z)  \I{z\in \phi(C_0)} +  \int \dz p_T^\phi(z)f^\phi_S(z) \I{z\in \phi(C_0)} \\
& = \int \dx f(x) p_T(x) \I{\phi(x)\in \phi(C_0)} +  \int \dz p_T^\phi(z)f^\phi_S(z) \I{z\in \phi(C_0)}\\
& = \int \dx p_T(x) \I{f(x)=1, \phi(x)\in \phi(C_0)} +  \int \dz p_T^\phi(z)f^\phi_S(z) \I{z\in \phi(C_0)} \addeq\label{eq:gb-decomp1-part1}
\end{align*}
Similarly, applying $\abs{f^\phi_T(z)-f^\phi_S(z)}=\abs{(1-f^\phi_T(z))-(1-f^\phi_S(z))} \le (1-f^\phi_T(z)) + (1-f^\phi_S(z))$ to the second part of \eqref{eq:gb-decomp1} gives
\begin{align*}
& \int \dz p_T^\phi(z) \abs{f^\phi_T(z)-f^\phi_S(z)} \I{z\in \phi(C_1)} \\
& \le \int \dz p_T^\phi(z) (1-f^\phi_T(z))  \I{z\in \phi(C_1)}  +  \int \dz p_T^\phi(z)(1-f^\phi_S(z)) \I{z\in \phi(C_1)} \\
& = \int \dz p_T^\phi(z) \left(1-\int \dx \phi_T(x|z)f(x)\right)  \I{z\in \phi(C_1)}  +  \int \dz p_T^\phi(z)(1-f^\phi_S(z)) \I{z\in \phi(C_1)} \\
& = \int \dx (1-f(x)) \int \dz p_T^\phi(z)  \phi_T(x|z)  \I{z\in \phi(C_1)} +  \int \dz p_T^\phi(z)(1-f^\phi_S(z)) \I{z\in \phi(C_1)} \\
& = \int \dx (1-f(x)) p_T(x) \I{\phi(x)\in \phi(C_1)} +  \int \dz p_T^\phi(z)(1-f^\phi_S(z)) \I{z\in \phi(C_1)}\\
& = \int \dx p_T(x) \I{f(x)=0, \phi(x)\in \phi(C_1)} +  \int \dz p_T^\phi(z)(1-f^\phi_S(z)) \I{z\in \phi(C_1)} \addeq\label{eq:gb-decomp1-part2}
\end{align*}
Combining the second part of \eqref{eq:gb-decomp1-part1} and the second part of \eqref{eq:gb-decomp1-part2}
\begin{align*}
& \int \dz p_T^\phi(z)f^\phi_S(z) \I{z\in \phi(C_0)} +   \int \dz p_T^\phi(z)(1-f^\phi_S(z)) \I{z\in \phi(C_1)} \\
& = \int \dz \frac{p_T^\phi(z)}{p_S^\phi(z)} p_S^\phi(z)f^\phi_S(z) \I{z\in \phi(C_0)}\left(\I{z\in B}+\I{z\in B^c} \right) \\
& +   \int \dz \frac{p_T^\phi(z)}{p_S^\phi(z)} p_S^\phi(z)(1-f^\phi_S(z)) \I{z\in \phi(C_1)}\left(\I{z\in B}+\I{z\in B^c} \right) \\
& \le (1+\beta)\int \dz p_S^\phi(z)f^\phi_S(z) \I{z\in \phi(C_0)} + (1+\beta)\int \dz p_S^\phi(z)(1-f^\phi_S(z)) \I{z\in \phi(C_1)} \\
& + \int \dz p_T^\phi(z) \I{z\in B^c} \left(\I{z\in \phi(C_0)}+\I{z\in \phi(C_1)} \right) \\
& \le (1+\beta)\int \dx p_S(x) \I{f(x)=1, \phi(x)\in \phi(C_0)} + (1+\beta)\int \dx p_S(x) \I{f(x)=0, \phi(x)\in \phi(C_1)} + p_T(B^c) \\
& \le (1+\beta)\int \dx p_S(x) \left(\I{f(x)=1, \phi(x)\in \phi(C_0) \lor f(x)=0, \phi(x)\in \phi(C_1)}\right) + \delta_1 \addeq\label{eq:gb-comb1}
\end{align*}
For $i\in \{0, 1\}$ if $x\in C_i$ then $f(x)=i$ and $\phi(x)\in C_i$. So if $f(x)=1, \phi(x)\in \phi(C_0)$ or $f(x)=0, \phi(x)\in \phi(C_1)$ holds we must have $x\notin C_0\cup C_1$. Therefore, following \eqref{eq:gb-comb1} gives
\begin{align*}
& \int \dz p_T^\phi(z)f^\phi_S(z) \I{z\in \phi(C_0)} +   \int \dz p_T^\phi(z)(1-f^\phi_S(z)) \I{z\in \phi(C_1)} \\
& \le (1+\beta)\int \dx p_S(x) \I{x\notin C_0\cup C_1} + \delta_1 \\
& = (1+\beta) (1-p_S(C_0\cup C_1)) + \delta_1 \\
& \le (1+\beta)\delta_2 + \delta_1\addeq\label{eq:gb-comb1-2}
\end{align*}
Now looking at the first part of \eqref{eq:gb-decomp1-part1} and the first part of \eqref{eq:gb-decomp1-part2}
\begin{align*}
& \int \dx p_T(x) \I{f(x)=1, \phi(x)\in \phi(C_0)} + \int \dx p_T(x) \I{f(x)=0, \phi(x)\in \phi(C_1)} \\
& = \int \dx p_T(x) \I{f(x)=1, \phi(x)\in \phi(C_0), x\in C_T} + \int \dx p_T(x) \I{f(x)=1, \phi(x)\in \phi(C_0), x\notin C_T} \\
& + \int \dx p_T(x) \I{f(x)=0, \phi(x)\in \phi(C_1), x\in C_T} + \int \dx p_T(x) \I{f(x)=0, \phi(x)\in \phi(C_1), x\notin C_T} \\
& \le \int \dx p_T(x) \left(\I{f(x)=1, \phi(x)\in \phi(C_0), x\in C_T} + \I{f(x)=0, \phi(x)\in \phi(C_1), x\in C_T}\right) + p_T(C_T^c) \\
& \le \int \dx p_T(x) \I{x\in C_T}\I{f(x)=1, \phi(x)\in \phi(C_0)\lor f(x)=0, \phi(x)\in \phi(C_1)} + \delta_3 \,. \addeq\label{eq:gb-comb1-3}
\end{align*}
Next we show that the first part of \eqref{eq:gb-comb1-3} is $0$. 
Recall that $\phi(C_T) \subset \phi(C_0 \cup C_1)$ and if $x\in C_T$ there exists 
$x'\in C_T \cap (C_0\cup C_1)$ with a sequence of points $x_0,x_1,...,x_m \in C_T$ such that $x_0=x$, $x_m=x'$, $f(x)=f(x')$ and $d_\rsp(x_{i-1}, x_i)< \frac{\Delta}{L}$ for all $i=1,..., m$.
So for $x\in C_T$ and $f(x)=i$,
we pick such $x'$.
Since $\phi$ is $L$-Lipschitz and $\phi(C_T) \subset \phi(C_0 \cup C_1)$ 
we have $\phi(x_0),\phi(x_1),...,\phi(x_m) \in  \phi(C_0 \cup C_1)$ and $d_\lsp(\phi(x_{i-1}), \phi(x_i))< \Delta$ for all $i=1,..., m$. 
Applying the fact that $\inf_{z_0\in \phi(C_0), z_1\in \phi(C_1)} d_\lsp(z_0, z_1) \ge \Delta > 0$ we know that if $\phi(x)=\phi(x_0)\in \phi(C_j)$ for some $j\in \{0, 1\}$
then $\phi(x')=\phi(x_m)\in \phi(C_j)$. 
From $x'\in C_0\cup C_1$ and $f(x')=f(x)=i$ we have $\phi(x')\in \phi(C_i)$. 
Since $C_0\cap C_1=\emptyset$ we can conclude $i=j$ and thus $\phi(x)\in \phi(C_i)$ if $f(x)=i$ for any $x\in C_T$. Therefore, if $x\in C_T$, neither $f(x)=1, \phi(x)\in \phi(C_0)$ nor $f(x)=0, \phi(x)\in \phi(C_1)$ can hold. Hence the first part of \eqref{eq:gb-comb1-3} is $0$.
 
So far by combining \eqref{eq:gb-comb1-2} and \eqref{eq:gb-comb1-3} we have shown that the sum of \eqref{eq:gb-decomp1-part1} and \eqref{eq:gb-decomp1-part2} (which are the first two parts of \eqref{eq:gb-decomp1}) can be upper bounded by $\delta_1+(1+\beta)\delta_2+\delta_3$. For the third part of \eqref{eq:gb-decomp1} we have 
\begin{align*}
& \int \dz p_T^\phi(z) \abs{f^\phi_T(z)-f^\phi_S(z)} \I{z\in \left(\phi(C_0)\cup\phi(C_1)\right)^c} \\
& \le \int \dz p_T^\phi(z) \I{z\in \left(\phi(C_0)\cup\phi(C_1)\right)^c} \\
& = \int \dz  \frac{p_T^\phi(z)}{p_S^\phi(z)}p_S^\phi(z) \I{z\in \left(\phi(C_0)\cup\phi(C_1)\right)^c}\left(\I{z\in B}+\I{z\in B^c}\right) \\
& \le \int \dz  \frac{p_T^\phi(z)}{p_S^\phi(z)}p_S^\phi(z) \I{z\in \left(\phi(C_0)\cup\phi(C_1)\right)^c}\I{z\in B} + \int \dz p_T^\phi(z)\I{z\in B^c} \\
& \le (1+\beta) \int \dz p_S^\phi(z) \I{z\in \left(\phi(C_0)\cup\phi(C_1)\right)^c} + \delta_1  \\
& = (1+\beta) \left(1-\int \dz p_S^\phi(z) \I{z\in \phi(C_0)\cup\phi(C_1)}\right) + \delta_1 \\
& = (1+\beta)  \left(1-\int \dx p_S(x) \I{x\in \phi^{-1}\left(\phi(C_0)\cup\phi(C_1)\right)}\right)+ \delta_1 \\
& = (1+\beta)  \left(1 - p_S\left(\phi^{-1}\left(\phi(C_0)\cup\phi(C_1)\right)\right) \right) + \delta_1 \\
& \le (1+\beta)  \left(1 - p_S\left(C_0\cup C_1\right) \right) + \delta_1 \\
& \le (1+\beta)\delta_2 + \delta_1 \,. \addeq\label{eq:gb-decomp1-part3}
\end{align*} 
Putting \eqref{eq:gb-decomp1-part3} into \eqref{eq:gb-decomp1} gives
\begin{align*}
\int \dz p_T^\phi(z) \left(\risk_T(z; \phi, h)- \risk_S(z; \phi, h)\right) \le 2\delta_1 + 2(1+\beta)\delta_2 + \delta_3 \,. \addeq\label{eq:main-part2}
\end{align*}

Plugging \eqref{eq:main-part1} and \eqref{eq:main-part2} into \eqref{eq:tar-err-decomp} gives the result of Theorem~\ref{thm:main} under Constructions~\ref{cons:main} and ~\ref{cons:conn}.

It remains to show that Assumption~\ref{assumption:main} implies the existence of a Construction~\ref{cons:conn}. To prove this, we first write $\phi(C_T)\subset \phi(C_0\cup C_1)$ as $C_T\subset \phi^{-1}(\phi(C_0\cup C_1))$. By Construction~\ref{cons:main} we have $p_S(C_0\cup C_1)\ge 1-\delta_2$. From \eqref{eq:gb-decomp1-part3} we have
\begin{align*}
& p_T\left(\phi^{-1}(\phi(C_0\cup C_1)) \right) = \int \dx p_T(x) \I{x \in \phi^{-1}(\phi(C_0\cup C_1)) } \\
& = \int \dz p_T^\phi(z) \I{z \in \phi(C_0\cup C_1) }  \ge (1+\beta)\delta_2 + \delta_1 \,.
\end{align*}
Setting $B_S = C_0\cup C_1$ and $B_T = \phi^{-1}(\phi(C_0\cup C_1)$ in Assumption~\ref{assumption:main} gives a construction of Construction~\ref{cons:conn}, thus concluding the proof.

\end{proof}

\begin{proof}[Proof of Corollary~\ref{coro:simp}]
Based on the statement of Corollary~\ref{coro:simp} it is obvious that 
Construction~\ref{cons:main} can be made with $\delta_1=0$, $\delta_2=0$ and a finitely large $L$. 
(Here we implicitly assume that $\phi$ is bounded on $\rsp$). 
It remains to show that Assumption~\ref{assumption:main} holds with $\delta_3=0$. 
As $\delta_1=\delta_2=0$, 
any $B_S$ and $B_T$ will be supersets of $\supp{p_S}$ and $\supp{p_T}$ respectively. 
So it sufficies to consider $B_S=\supp{p_S}$ and $B_T=\supp{p_T}$. 

Now we verify that $C_T=\supp{p_T}$ satisfies the requirements in Assumption~\ref{assumption:main}. According to Assumption~\ref{assumption:simp}, for any $x\in \supp{p_T}$, there must exist $S_{T,i,j}$ such that $x \in S_{T,i,j}$, $S_{T,i,j}$ is connected, $f(x')=i$ for all $x'\in S_{T,i,j}$ and $S_{T,i,j}\cap \supp{p_S}\ne \emptyset$. Pick $x' \in S_{T,i,j}\cap \supp{p_S}$. Such $x'$ satisfies $x'\in C_T \cap B_S$ with our choice of $C_T$ and $B_S$. Since $S_{T,i,j}$ is connected we can find a sequence of points $x_0,...,x_m \in S_{T,i,j}$ with $x_0=0$, $x_m=x'$ and $d_\rsp(x_{i-1}, x_i)< \epsilon$ for any $\epsilon > 0$. As $S_{T,i,j}$ is label consistent we have $f(x)=f(x')$. Picking $\epsilon = \frac{\Delta}{L}$ concludes the fact that $C_T=\supp{p_T}$ satisfies the requirements in Assumption~\ref{assumption:main}.

Since $p_T(\supp{p_T})=1$ we have $\delta_3=0$. As a result, $\err_T(\phi, h) \le (1+\beta)\err_S(\phi,h)$ holds according to Theorem~\ref{thm:main}, which concludes the proof of Corollary~\ref{coro:simp}.

\end{proof}

\begin{proof}[Derivation of \eqref{eq:f-dual}]
The Fenchel Dual of $\bar{f}_\beta(u)$ can be written as
\begin{align*}
\bar{f}_\beta^*(t) & = 
\begin{cases}
tf'^{-1}(t) - \bar{f}_\beta(f'^{-1}(t))  &\text{ if } t \le f'(\frac{1}{1+\beta}) \,,\\
+\infty &\text{ if } t > f'(\frac{1}{1+\beta}) \,.
\end{cases} \\
& = 
\begin{cases}
tf'^{-1}(t) - f(f'^{-1}(t)) + C  &\text{ if } t \le f'(\frac{1}{1+\beta}) \,,\\
+\infty &\text{ if } t > f'(\frac{1}{1+\beta}) \,.
\end{cases}  \\
& = 
\begin{cases}
f^*(t) + C_{f,\beta}  &\text{ if } t \le f'(\frac{1}{1+\beta}) \,,\\
+\infty &\text{ if } t > f'(\frac{1}{1+\beta}) \,.
\end{cases} \,,
\end{align*}
where $C_{f,\beta} = f(\frac{1}{1+\beta}) - f'(\frac{1}{1+\beta})\frac{1}{1+\beta} + f'(\frac{1}{1+\beta})$.

Therefore, the modified $\bar{f}_\beta$-divergence can be written as 
\begin{align*}
D_{f,\beta}(p,q) = \sup_{T:\lsp \mapsto \dom{f^*} \cap (-\infty, f'(\frac{1}{1+\beta})]} \EE{z\sim q}{T(z)} - \EE{z\sim p}{f^*(T(z))} - C_{f,\beta}\,.
\end{align*}

\end{proof}

\begin{proof}[Derivation of \eqref{eq:f-dual-gan}]

According to \citet{nowozin2016f}, the GAN objecitve uses $f(u)=u\log u - (1+u) \log (1+u)$. Hence $f^*(t) = - \log (1-e^t)$, $f'(u) = \log\frac{u}{u+1}$ and $f'(\frac{1}{1+\beta}) = \log\frac{1}{2+\beta}$. 
So we need to parameterize $T:\lsp \mapsto \big(-\infty, \log\frac{1}{2+\beta}\big]$. $T(z) = \log \frac{g(z)}{2+\beta}$ with $g(z)\in (0, 1]$ satisfies the range constraint for $T$. Plugging  $T(z) = \log \frac{g(z)}{2+\beta}$ into \eqref{eq:f-dual} gives the result of \eqref{eq:f-dual-gan}.

\end{proof}



\section{Experiment Details}

\textbf{Synthetic datasets} For source distribution, we sample class $0$ from $\Normal([-1, -0.3], diag(0.1, 0.4))$ and class $1$ from $\Normal([1, 0.3], diag(0.1, 0.4))$. For target distribution, we sample class $0$ from $\Normal([-0.3, -1], diag(0.4, 0.1))$ and class $1$ from $\Normal([0.3, 1], diag(0.4, 0.1))$. For label classifier, we use a fully-connect neural net with 3 hidden layers $(50, 50, 2)$ and the latent space is set as the last hidden layer. For domain classifier (critic) we use a fully-connect neural net with 2 hidden layers $(50, 50)$.

\textbf{Image datasets} For MNIST we subsample 2000 data points and for USPS we subsample 1800 data points. The subsampling process depends on the given label distribution (e.g. shift or no-shift). For label classifier, we use LeNet and the latent space is set as the last hidden layer. For domain classifier (critic) we use a fully-connect neural net with 2 hidden layers $(500, 500)$.

In all experiments, we use $\lambda=1$ in the objective \eqref{eq:obj-general} and ADAM with learning rate 0.0001 and $\beta_1=0.5$ as the optimizer. We also apply a l2-regularization on the weights of $\phi$ and $h$ with coefficient $0.001$.

\textbf{More discussion on synthetic experiments.} The only unexcepted failure is WDANN1-$2$, which achieves only 20\% accuracy in 2-out-of-5 runs. Looking in to the low accuracy runs we found that the l2-norm of the encoder weights is clearly higher than the successful runs. Large l2-norm of weights in $\phi$ likely results in a high Lipschitz constant $L$, which is undesirable according to our theory. We only implemented l2-regularization to encourage Lipschitz continuity of the encoder $\phi$, which might be insufficient. How to enforce Lipschitz continuity  of a neural network is still an open question. Trying more sophisticated approaches for Lipschitz continuity can a future direction.

\textbf{Choice of $\beta$.} Since a good value of $\beta$ may depend on the knowledge of target label distribution which is unknown, we experiment with different values of $\beta$. Empirically we did not find any clear pattern of correlation between value of $\beta$ and performance as long as it is big enough to accommodate label distribution shift so we would leave it as an open question. In practice we suggest to use a moderate value such as $2$ or $4$, or estimate based on prior knowledge of target label distribution.